\documentclass{article}
\usepackage{spconf,amsmath,amssymb,amsthm,graphicx,booktabs}
\usepackage{multirow}
\usepackage{xcolor}
\usepackage{comment}
\usepackage{subcaption}
\usepackage{url}
\newtheorem{proposition}{Proposition}
\newcommand{\E}{\mathbb E}
\newcommand{\norm}[1]{\left\lVert#1\right\rVert}
\newcommand{\FB}{\textsc{FlowBudget}}

\title{FLOWBUDGET: TRAINING-FREE SAMPLING-BUDGET ESTIMATION FOR FLOW MATCHING MODELS}
\name{Zahra Dehghani$^{1,2,3,4}$ \qquad Pablo Piantanida$^{3,4,5}$ \qquad Mohammadhadi Shateri$^{1,2,3}$}
\address{$^{1}$ÉTS Montréal, $^{2}$LIVIA, $^{3}$ILLS, $^{4}$Mila -- Quebec AI Institute,\\
$^{5}$CNRS, CentraleSupélec -- Université Paris-Saclay}

\begin{document}
\setlength{\textfloatsep}{5pt}
\setlength{\floatsep}{5pt}
\setlength{\intextsep}{5pt}
\maketitle

\begin{abstract}
Sampling from continuous flow-based generators such as flow matching models requires numerical ODE or SDE integration, making inference cost depend on the number of function evaluations (NFEs). In practice, the sampling budget is inherited from prior work or selected through expensive FID-versus-NFE sweeps, using FID to assess image-generation quality, requiring sample generation and evaluation at multiple budgets. This makes sampling-budget selection dependent on repeated empirical evaluation. We therefore ask a complementary question to solver design and timestep allocation: for a fixed pretrained model and solver, what numerical resolution is sufficient? We introduce \FB, a training-free estimator of solver-relative sampling complexity. Using simulation-free interpolation probes of the learned dynamics, \FB{} normalizes the leading one-step discretization error by the motion, estimates its coefficient from interpolation probes, and maps a numerical tolerance to a step budget using the solver order. We derive estimators for numerical solvers in flow~matching sampling, namely Euler and Heun for ODE sampling and additive-noise Euler--Maruyama for SDE sampling, and connect the ODE coefficients to terminal integration error. Experiments on SiT, FlowDCN and Rectified Flow models, spanning flow-model architectures and formulations, test whether \FB{} can estimate the  NFE budget without generating samples at multiple budgets or performing an FID-versus-NFE sweep. The code is available at \url{https://github.com/Yasaman-dt/FlowBudget}.
\end{abstract}


\section{Introduction}
\label{sec:intro}
Flow matching models learn a time-dependent velocity field whose integration transports a reference law to data~\cite{lipman2023flow,liu2023flow,albergo2025stochastic}, and are used across image~\cite{esser2024scaling}, audio~\cite{le2023voicebox,vyas2023audiobox}, and video~\cite{kong2024hunyuanvideo} generation. 
At inference, generation requires numerical integration of ordinary differential equations (ODEs) or stochastic differential equations (SDEs) parameterized by the learned model~\cite{chen2018neural,anderson1982reverse,song2021scorebased,ma2024sit}.
Because each integration step requires one or more neural-network evaluations, the solver and its number of function evaluations (NFEs) directly determine sampling cost. Few NFEs can produce substantial discretization error, whereas increasing NFEs reduces this error according to the convergence order of the chosen solver~\cite{hairer1993solving,lu2022dpm,shaul2024bespoke}. In practice, however, the number of NFEs is often inherited from a reference implementation, e.g., 250 NFEs~\cite{ma2024sit,wang2024exploring}, or selected by trial and error based on generation quality. 

Generation error can arise from both approximation error in the learned velocity field and numerical error introduced by the solver. For a fixed pretrained model, the former is fixed, so we focus on the numerical integration error.  Recent methods improve sampling efficiency by changing the solver or by allocating a fixed budget nonuniformly~\cite{lu2022dpm,shaul2024bespoke,sabour2024align,gupta2026sharpen,li2026cat,bondar2026velocity}. In particular, SharpEuler uses trajectory acceleration---the leading Euler local-error quantity---to decide \emph{where} a prescribed Euler budget should be spent~\cite{gupta2026sharpen}. Our problem is complementary: we keep the pretrained model and solver fixed and ask \emph{how large the total budget should be}. Thus, the novelty is not the use of local truncation error itself, but its use to define and estimate a solver-relative numerical sampling complexity before performing a quality-versus-NFE sweep.

We introduce \FB{}, which associates each pretrained model--solver pair with a scale-normalized numerical-complexity coefficient. Estimated from simulation-free interpolation probes, this coefficient combines with the solver order to map a prescribed relative numerical tolerance directly to a sampling-step budget.  
Our contributions are threefold. First, we formulate fixed-solver budget estimation as a numerical-error problem and derive a general scaling law relating a stabilized relative local-error tolerance to the required step budget. Second, we instantiate the construction for Euler and Heun ODE sampling and additive-noise Euler--Maruyama sampling, and connect the Euler/Heun differential coefficients to terminal ODE integration error under standard stability assumptions. We also distinguish the population numerical-complexity coefficient from its Monte Carlo estimate and give a finite-probe concentration guarantee. Third, across pretrained SiT, Rectified Flow, and FlowDCN models, we test whether this numerical resolution yields useful quality--cost operating points without using FID-versus-NFE sweeps to construct the estimator.

\section{Background and Problem Formulation}
\label{sec:background}
\textbf{Background.}
Flow matching learns a time-dependent velocity field for transporting samples between noise and data distributions~\cite{lipman2023flow}. Given data $x_0\sim p_{\mathrm{data}}$ and Gaussian noise $\epsilon\sim\mathcal N(0,I)$, consider the interpolant $x_t=a_t x_0+b_t\epsilon$, $t\in[0,1]$, where the differentiable schedules satisfy $a_0=b_1=1$ and $a_1=b_0=0$. Throughout this work, we adopt the conventional linear schedules $a_t=1-t$ and $b_t=t$, as used in the literature, e.g., Rectified Flow~\cite{liu2023flow}, SiT~\cite{ma2024sit}, and FlowDCN~\cite{wang2024exploring}. The resulting interpolant is $x_t=(1-t)x_0+t\epsilon$, with conditional velocity $v_t=\frac{dx_t}{dt}=\epsilon-x_0$. A neural velocity field $v_\theta(x_t,t)$ is trained to predict this conditional velocity by minimizing $\mathcal L_{\mathrm{FM}}(\theta) =\mathbb E_{x_0,\epsilon,t} [\|v_\theta(x_t,t)-v_t\|_2^2]$, where $t\sim\mathcal U[0,1]$. 
At inference, generation starts from Gaussian noise $x_1\sim\mathcal N(0,I)$ and proceeds backward from $t=1$ to $t=0$. For $K$ uniform steps, we define the time grid $t_n=1-n/K$, $n=0,\ldots,K$, and the signed step size $h=t_{n+1}-t_n=-1/K$. We write $x_{t_n}$ for the numerical state at time $t_n$, initialized with $x_{t_0}=x_1$. For ODE sampling we integrate $\frac{dx_t}{dt}=v_\theta(x_t,t)$ using Euler or Heun~\cite{hairer1993solving}. Euler uses the update $x_{t_{n+1}}^{E}
=x_{t_n}+h\,v_\theta(x_{t_n},t_n)$, while Heun uses $
x_{t_{n+1}}^{H} = x_{t_n}+\frac{h}{2} \left[ v_\theta(x_{t_n},t_n) + v_\theta\!\left(
x_{t_n}+h\,v_\theta(x_{t_n},t_n),t_{n+1}
\right)
\right]$.
For SDE sampling, we consider
$dx_t=f_\theta(x_t,t)\,dt+\sqrt{w_t}\,dW_t$, where $W_t$ is a reverse-time Wiener process, $w_t\geq0$ is a prescribed time-dependent noise variance rate, and $f_\theta(x_t,t)
=v_\theta(x_t,t)-\frac{1}{2}w_t s_\theta(x_t,t)$ is the reverse-time drift. For the linear interpolant, the score is estimated from the learned velocity as $s_\theta(x_t,t)
=-\frac{x_t+(1-t)v_\theta(x_t,t)}{t}$~\cite{song2021scorebased,ma2024sit}. On the same time grid, Euler--Maruyama gives $x_{t_{n+1}}^{\mathrm{EM}}=
x_{t_n} + h \left[ v_\theta(x_{t_n},t_n)
- \frac{1}{2}w_{t_n}s_\theta(x_{t_n},t_n)
\right] + \sqrt{w_{t_n}|h|}\,\xi_n$, where $\xi_n\sim\mathcal N(0,I)$ are independent~\cite{higham2001algorithmic}. 

\noindent \textbf{Problem Formulation.}
Given a pretrained velocity field $v_\theta$, a fixed sampling dynamics and solver $S$, and a tolerance $\tau>0$, we estimate the smallest uniform step budget $K$ whose leading-order mean relative local error is at most $ \tau$. We define relative local error as the one-step discretization error divided by a stabilized displacement scale, and estimate its mean using data-noise interpolation probes constructed from a small data subset. This yields an asymptotic budget estimate without simulating sampling trajectories. 

\section{Methodology}
\label{sec:method}

\FB{} estimates the sampling budget from local properties of the dynamics evaluated at interpolation probes. We construct $x_t=(1-t)x_0+t\epsilon$ and denote $\mu$ as a simulation-free surrogate for the sampler-induced state--time law. For a probe $(x_t,t)$, consider a step of magnitude $|h|$ in the sampling direction. 
Let $L_S(x_t,t;h)$ be the one-step local error between exact and numerical updates starting from the same state, and let $D_S(x_t,t;h)$ denote the one-step displacement scale. For SDEs, these quantities use the conditional root-mean-square norm, with the exact and numerical updates sharing the same Brownian path. Their small-step expansions take the form
\begin{align}
L_S(x_t,t;h)&=\kappa_S|h|^{\alpha_S}q_S(x_t,t)+O(|h|^{\alpha_S+\delta}),\\
D_S(x_t,t;h)&=|h|^{\beta_S}d_S(x_t,t)+O(|h|^{\beta_S+\delta}),
\label{eq:local}
\end{align}
where $\kappa_S, \delta>0$, $q_S,d_S\geq0$, and $\alpha_S>\beta_S>0$. Coefficients $q_S$ and $d_S$ describe the local error and motion, while the exponents determine their dependence on the step size. 
To measure local numerical error relative to this motion, define
\begin{align}
{\cal E}_S(x_t,t;h)
&=\frac{L_S(x_t,t;h)}
{D_S(x_t,t;h)+\varepsilon |h|^{\beta_S}},\quad \varepsilon>0.
\label{eq:relative}
\end{align}
The stabilization scales with the displacement, preserving the leading step-size dependence. The corresponding normalized coefficient and its probe average are 
\begin{equation}
r_\theta^S(x_t,t)=\frac{q_S(x_t,t)}{d_S(x_t,t)+\varepsilon},\qquad
{\mathfrak C}_S=\E_\mu \left[r_\theta^S(x_t,t)\right]. \label{eq:C}
\end{equation}

\begin{proposition}[Budget scaling]
\label{prop:scaling}
Let $\gamma_S=\alpha_S-\beta_S$. Suppose that, uniformly on the support of the probe law $\mu$, we have $L_S =\kappa_S|h|^{\alpha_S}q_S+R_{L,S},
|R_{L,S}|\le M_L|h|^{\alpha_S+\delta}$ and $D_S=|h|^{\beta_S}d_S+R_{D,S},
 |R_{D,S}|\le M_D|h|^{\beta_S+\delta}$, for integrable leading coefficients. For sufficiently small $|h|$
\begin{equation}
\overline{\cal E}_S(K):=\E_\mu \left[{\cal E}_S(x_t,t;h)\right]
=\kappa_S{\mathfrak C}_S K^{-\gamma_S}
+O(|h|^{\gamma_S+\delta}).
\label{eq:scaling}
\end{equation}
Accordingly, solving $\kappa_S{\mathfrak C}_S K^{-\gamma_S}=\tau$ defines the population leading-order budget and its Monte Carlo plug-in estimate
\begin{align}
\widehat{\mathfrak C}_S&=\frac1N\sum_{i=1}^N r_\theta^S(x^{(i)}_{t^{(i)}},t^{(i)}), \widehat K_S(\tau)
=\left\lceil\left({\kappa_S\widehat{\mathfrak C}_S}/{\tau}\right)^{1/\gamma_S}\right\rceil.
\label{eq:budget}
\end{align}
\end{proposition}
\emph{Proof.}
Factor $|h|^{\beta_S}$ from the denominator in \eqref{eq:relative}. Since $d_S+\varepsilon\ge\varepsilon$ and $|R_{D,S}|/|h|^{\beta_S}\le M_D|h|^\delta$, for sufficiently small $|h|$ the latter term is at most $\varepsilon/2$. Hence the reciprocal denominator admits a uniform expansion around $d_S+\varepsilon$. Multiplication by the numerator expansion gives
$\mathcal E_S(x_t,t;h)=\kappa_S|h|^{\gamma_S}r_\theta^S(x_t,t)+O(|h|^{\gamma_S+\delta})$ uniformly. Averaging under $\mu$ and setting $|h|=K^{-1}$ yields \eqref{eq:scaling}; inversion of leading term 
and replacing ${\mathfrak C}_S$ by empirical mean gives \eqref{eq:budget}. \hfill$\square$

For $N$ independent probes, $\widehat{\mathfrak C}_S$ is a Monte Carlo estimate of $\mathfrak C_S$. If $0\le r_\theta^S(x_t,t)\le R_S$ $\mu$-a.s., Hoeffding's inequality gives, with probability at least $1-\eta$,
\begin{equation}
|\widehat{\mathfrak C}_S-\mathfrak C_S|\le R_S\sqrt{{\log(2/\eta)}/{2N}},
\label{eq:concentration}
\end{equation}
which induces a confidence interval for the unrounded leading-order budget through the monotonicity of~\eqref{eq:budget}.
\subsection{Euler and Heun}
Along an ODE trajectory, with all derivatives evaluated at $(x_t,t)$, write
$v_\theta=v_\theta(x_t,t)$ and define
\begin{align}
a_\theta=&\partial_t v_\theta+J_xv_\theta v_\theta,\quad
g_\theta=J_xv_\theta a_\theta,\\
c_\theta=&D^2_{(x_t,t)}v_\theta[(v_\theta,1),(v_\theta,1)].
\end{align}
For a solver $S$, initialized from the exact state $x_{t_n}=x_t$, define the one-step error as
$e_S(x_t,t;h)
= x_{t+h}-x_{t_{n+1}}^{S}$.
Taylor expansion along the exact trajectory gives
\begin{equation}
x_{t+h}=x_t+hv_\theta+(h^2/2)a_\theta+
(h^3/6)(c_\theta+g_\theta)+O(h^4),
\end{equation}
where $\frac{d}{dt}a_\theta=c_\theta+g_\theta$ and
$ v_\theta(x_t+h v_\theta,t+h)=v_\theta+h a_\theta+({h^2}/{2})c_\theta+O(h^3)$.
Hence Euler and Heun have respectively
\begin{align}
e_E=({h^2}/{2})a_\theta+O(h^3),~\hspace{0.1cm}
e_H =({h^3}/{12})(2g_\theta-c_\theta)+O(h^4),
\end{align}
which yields
\begin{align}
{\mathfrak C}_E&=\E_\mu\frac{\norm{a_\theta}}{\norm{v_\theta}+\varepsilon},
&
\widehat K_E&=\left\lceil\widehat{\mathfrak C}_E/(2\tau)\right\rceil,
\label{eq:euler}\\
{\mathfrak C}_H&=\E_\mu\frac{\norm{2g_\theta-c_\theta}}
{\norm{v_\theta}+\varepsilon},
&
\widehat K_H&=\left\lceil
\sqrt{\widehat{\mathfrak C}_H/(12\tau)}\right\rceil .
\label{eq:heun}
\end{align}
Euler uses one NFE per step and Heun two, so $\widehat{\mathrm{NFE}}_E=\widehat K_E$ and $\widehat{\mathrm{NFE}}_H=2\widehat K_H$. Here the generic coefficients in \eqref{eq:local} specialize to $q_E=\norm{a_\theta}$ and $q_H=\norm{2g_\theta-c_\theta}$. The Euler coefficient is trajectory acceleration, also used by recent fixed-budget scheduling methods such as SharpEuler~\cite{gupta2026sharpen}; in contrast, \FB{} uses it to estimate the total budget. Heun depends on a solver-specific third-order combination rather than acceleration alone.
The same numerator coefficients also control terminal ODE discretization error. Under standard Lipschitz stability and uniform local expansions, with backward step magnitude $|h|=1/K$, error propagation over a unit interval gives
\begin{align}
\norm{x_0-x_E^{(K)}}&\le \frac{e^L}{2K^2}\sum_{n=0}^{K-1}\norm{a_\theta(x_{t_n},t_n)}+O(K^{-2}),\\
\norm{x_0-x_H^{(K)}}&\le \frac{e^L}{12K^3}\sum_{n=0}^{K-1}\norm{2g_\theta-c_\theta}(x_{t_n},t_n)+O(K^{-3}).
\label{eq:globalcompact}
\end{align}
If additionally $\norm{v_\theta(x_t,t)}+\varepsilon\le V$, Riemann-sum convergence yields
$\norm{x_0-x_E^{(K)}}\le e^L V\mathfrak C_{E,\mathrm{traj}}/(2K)+o(K^{-1})$ and
$\norm{x_0-x_H^{(K)}}\le e^L V\mathfrak C_{H,\mathrm{traj}}/(12K^2)+o(K^{-2})$, where
$\mathfrak C_{S,\mathrm{traj}}=\int_0^1 r_\theta^S(x_t,t)dt$.
Thus the normalized coefficients control terminal error along exact sampler trajectories under this mild velocity bound. \FB{} replaces these trajectory-time averages by interpolation-probe averages, an approximation evaluated empirically below. No implication from numerical error to FID is assumed.

\subsection{Euler--Maruyama}
We consider the additive-noise SDE
$dX_t=f_\theta(X_t,t)dt+G(t)dW_t$.  The following construction is a
\emph{strong local-error} statement: the exact and numerical one-step
solutions are coupled with the same Brownian path.  For
$\Delta t>0$, define
$L_{\rm EM}(x_t,t;\Delta t)
=\left(\E[\norm{X_{t+\Delta t}-X^{\rm EM}_{t+\Delta t}}^2\mid X_t=x
]\right)^{1/2}.$
We formulate the local analysis in the sampler's increasing computational-time coordinate, avoiding sign ambiguities from the backward physical-time parameterization. Let
$A=J_xf_\theta(x_t,t)G(t)$ and $B=\partial_tG(t)$, and let
$\norm{\cdot}_F$ denote the Frobenius norm. The additive-noise It\^o--Taylor expansion coupled to the same Brownian path yields the leading defect
$A\int_0^{\Delta t}(W_s-W_0)ds+B\int_0^{\Delta t}s\,dW_s+O_{L^2}(\Delta t^2)$.
Writing $I_1=\int_0^{\Delta t}(W_s-W_0)ds$ and $I_2=\int_0^{\Delta t}s\,dW_s$, stochastic Fubini and It\^o isometry give
$\E I_1^2=\E I_2^2=\Delta t^3/3$ and $\E[I_1I_2]=\Delta t^3/6$. Therefore
\begin{equation}
L_{\rm EM}(x_t,t;\Delta t)=\frac{\Delta t^{3/2}}{\sqrt3}
\big(\norm A_F^2+\norm B_F^2+\langle A,B\rangle_F\big)^{1/2}+O(\Delta t^2).
\label{eq:em-local}
\end{equation}
The coefficient is nonnegative because
$\norm A_F^2+\norm B_F^2+\langle A,B\rangle_F
=\tfrac12(\norm A_F^2+\norm B_F^2+\norm{A+B}_F^2)$.
This $O(\Delta t^{3/2})$ one-step RMS defect is consistent with the
order-one strong convergence of Euler--Maruyama (EM) for sufficiently regular
additive-noise SDEs; it should not be confused with the generic
order-$1/2$ strong rate for multiplicative noise.
Define the EM displacement scale by
$D_{\rm EM}(x_t,t;\Delta t)=\big(\E[\norm{X^{\rm EM}_{t+\Delta t}-x}^2\mid X_t=x]\big)^{1/2}$.
On a probe-time domain $\mathcal T=[t_{\min},t_{\max}]\subset(0,1)$ where $G$ and $\partial_tG$ are bounded and $\inf_{t\in\mathcal T}\norm{G(t)}_F>0$, the diffusion term dominates and
$D_{\rm EM}=\Delta t^{1/2}\norm G_F+O(\Delta t)$.  Hence
$\alpha_{\rm EM}=3/2$, $\beta_{\rm EM}=1/2$, and
$\gamma_{\rm EM}=1$.  Proposition~\ref{prop:scaling} gives
\begin{equation}
{\mathfrak C}_{\rm EM}
=\E_\mu\frac{
(\norm A_F^2+\norm B_F^2+\langle A,B\rangle_F)^{1/2}}
{\norm G_F+\varepsilon},
~
\widehat K_{\rm EM}
=\left\lceil\frac{\widehat{\mathfrak C}_{\rm EM}}
{\sqrt3\,\tau}\right\rceil .
\label{eq:em}
\end{equation}
Accordingly, SDE probes are restricted to $t\in\mathcal T$; this excludes endpoints where $G$ vanishes or $\partial_tG$ becomes singular and the above normalization/expansion may fail.  Euler--Maruyama uses one drift/network
evaluation per step, so
$\widehat{\mathrm{NFE}}_{\rm EM}=\widehat K_{\rm EM}$.
Importantly, this estimator controls a strong numerical-resolution proxy;
we make no theoretical claim that its tolerance is calibrated to FID. In practice, \FB{} draws $N$ interpolation probes $x_t=(1-t)x_0+t\epsilon$, evaluates the solver coefficient in \eqref{eq:euler}, \eqref{eq:heun}, or \eqref{eq:em}, averages it to obtain $\widehat{\mathfrak C}_S$, and applies \eqref{eq:budget}; the NFE multiplier is one for Euler/EM and two for Heun.

\section{Experiments}
\label{sec:experiments}

\noindent \textbf{Experimental Setup.}
We evaluate \FB{} across six pretrained models: SiT-XL/2~\cite{ma2024sit} and FlowDCN-XL/2~\cite{wang2024exploring} on ImageNet~\cite{deng2009imagenet}, and RF-UNet, 1-RF, 2-RF, and 3-RF~\cite{liu2023flow} on CIFAR-10~\cite{krizhevsky2009learning}, using the solvers from their original works and additionally with Heun for FlowDCN. For all model--solver pairs, \FB{} estimates the step budget from $2$K interpolation probes using numerical tolerance $\tau=0.010$ for the ODE solvers and $\tau=0.015$ for the SDE solver, with $\varepsilon=10^{-8}$, and converts the step budget to the $\mathrm{NFE}$ budget. We use one tolerance within each dynamics class rather than tuning it per model. The experiment tests whether a fixed numerical tolerance yields useful budgets across architectures, not whether the tolerance is universally calibrated to perceptual quality. The predicted $\mathrm{NFE}$ budgets are essentially unchanged when setting $\varepsilon=0$. For each seed, FID is computed from $10$K samples at its predicted budget. Each seed uses an independent probe set, so the reported variation reflects Monte Carlo variability in the budget estimate. Each run used an NVIDIA A100 40GB GPU. 

\begin{table}[t]
\centering
\caption{\FB{}-predicted budgets and corresponding FID, reported as mean $\pm$ std over three seeds. Reference FID uses  $\mathrm{NFE}=250$ for SiT-XL/2, FlowDCN-XL/2, and RF-UNet, and $\mathrm{NFE}=1$ for 1-RF, 2-RF, and 3-RF.}
\label{tab:main_results}
\scriptsize
\setlength{\tabcolsep}{3pt}
\renewcommand{\arraystretch}{1.0}
\resizebox{\columnwidth}{!}{%
\begin{tabular}{c|c|c|c|c}
\toprule

Model
& Solver
& $\widehat{\mathrm{NFE}}$
& FID
& Reference FID \\

\midrule

SiT-XL/2 & ODE Euler
& $146\,{\scriptstyle \pm 3}$ & $16.60\,{\scriptstyle \pm 0.01}$
& $16.41$ \\

SiT-XL/2 & ODE Heun
& $58\,{\scriptstyle \pm 2}$ & $16.81\,{\scriptstyle \pm 0.21}$
& $16.56$ \\

SiT-XL/2 & SDE EM
& $210\,{\scriptstyle \pm 2}$ & $16.59\,{\scriptstyle \pm 0.16}$
& $16.06$ \\

\midrule

FlowDCN-XL/2 & ODE Euler
& $150\,{\scriptstyle \pm 2}$ & $15.53\,{\scriptstyle \pm 0.01}$ & $15.36$  \\

FlowDCN-XL/2 & ODE Heun
& $42\,{\scriptstyle \pm 0}$ & $15.94\,{\scriptstyle \pm 0}$ & $15.16$  \\

FlowDCN-XL/2 & SDE EM
& $178\,{\scriptstyle \pm 1}$ & $15.78\,{\scriptstyle \pm 0.22}$ & $15.42$ \\

\midrule

RF-UNet & ODE Euler
& $140\,{\scriptstyle \pm 9}$ & $6.79\,{\scriptstyle \pm 0.03}$
& $6.59$ \\

\midrule

1-RF & ODE Euler
& $148\,{\scriptstyle \pm 11}$ & $7.40\,{\scriptstyle \pm 0.05}$
& $372.51$ \\

2-RF & ODE Euler
& $19\,{\scriptstyle \pm 2}$ & $7.80\,{\scriptstyle \pm 0.02}$
& $16.41$ \\

3-RF & ODE Euler
& $15\,{\scriptstyle \pm 1}$ & $8.24\,{\scriptstyle \pm 0.01}$
& $11.63$\\

\bottomrule
\end{tabular}%
}
\end{table}

\noindent \textbf{Main Results.}
Fig.~\ref{fig:fid_vs_nfe} and Table~\ref{tab:main_results} show that \FB{} selects budgets near the low-FID plateau of the empirical curves. For SiT, FlowDCN, and RF-UNet, using a reference NFE of $250$, it reduces the sampling budget by $16.0$--$76.8\%$ for SiT-XL/2, $28.8$--$83.2\%$ for FlowDCN-XL/2.  For the same model, the predicted budget also depends strongly on the solver: Heun requires only $58$ versus $146$ $\mathrm{NFE}$s for SiT-XL/2 and $42$ versus $150$ for FlowDCN-XL/2 under the same ODE tolerance, consistent with its higher-order accuracy. The FID--NFE curves are used only retrospectively and they are not used for budget estimation. Overall, \FB{} identifies a solver-dependent quality--cost operating point rather than directly predicting FID saturation.

\begin{figure}[h]
    \centering

    \begin{subfigure}{\linewidth}
        \centering
        \includegraphics[width=\linewidth]{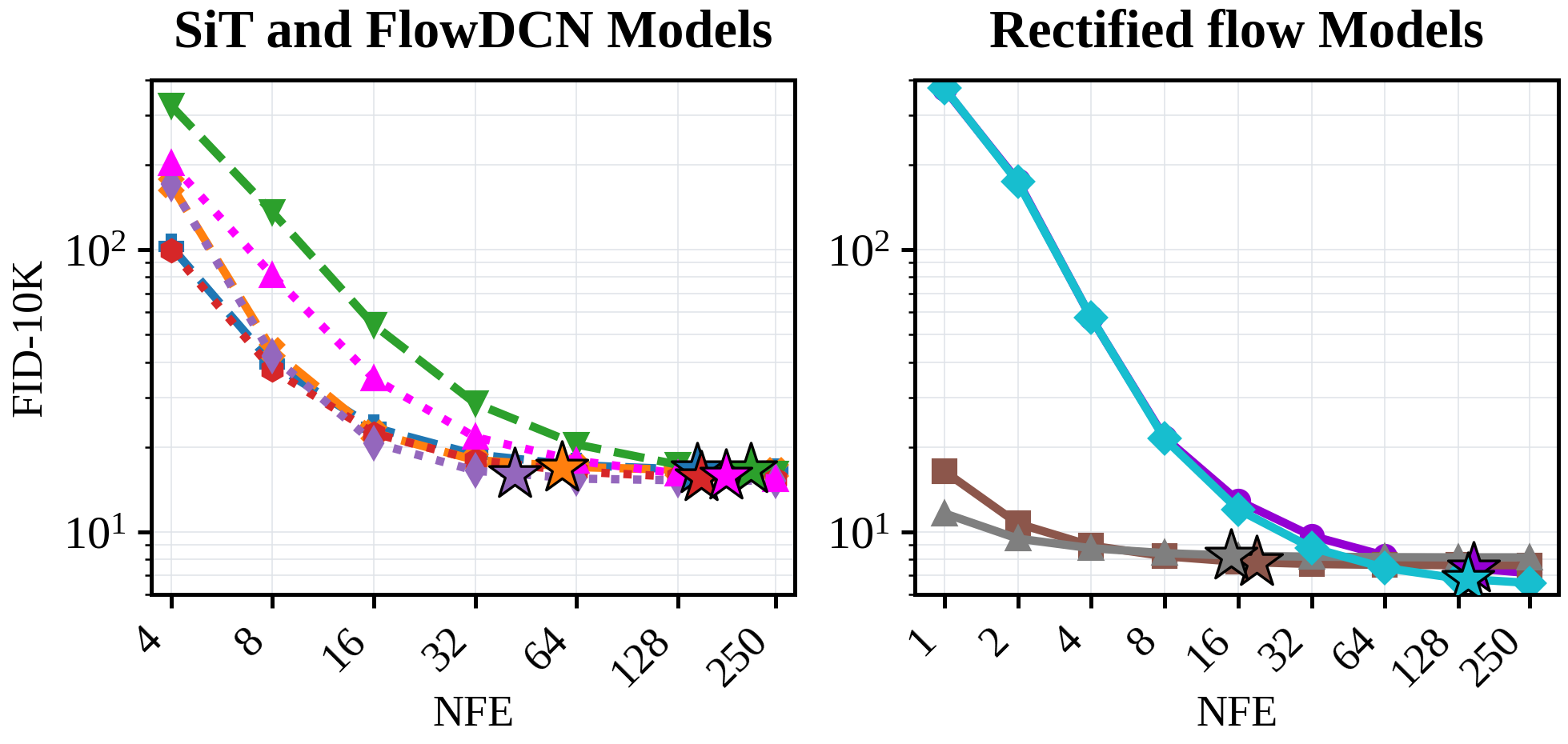}
        \vspace{-0.6cm}
        \caption{FID versus number of function evaluations (NFE), with stars
        indicating the \FB{}-predicted NFE budgets.}
        \label{fig:fid_vs_nfe}
    \end{subfigure}

    \begin{subfigure}{\linewidth}
        \centering
        \includegraphics[width=\linewidth]{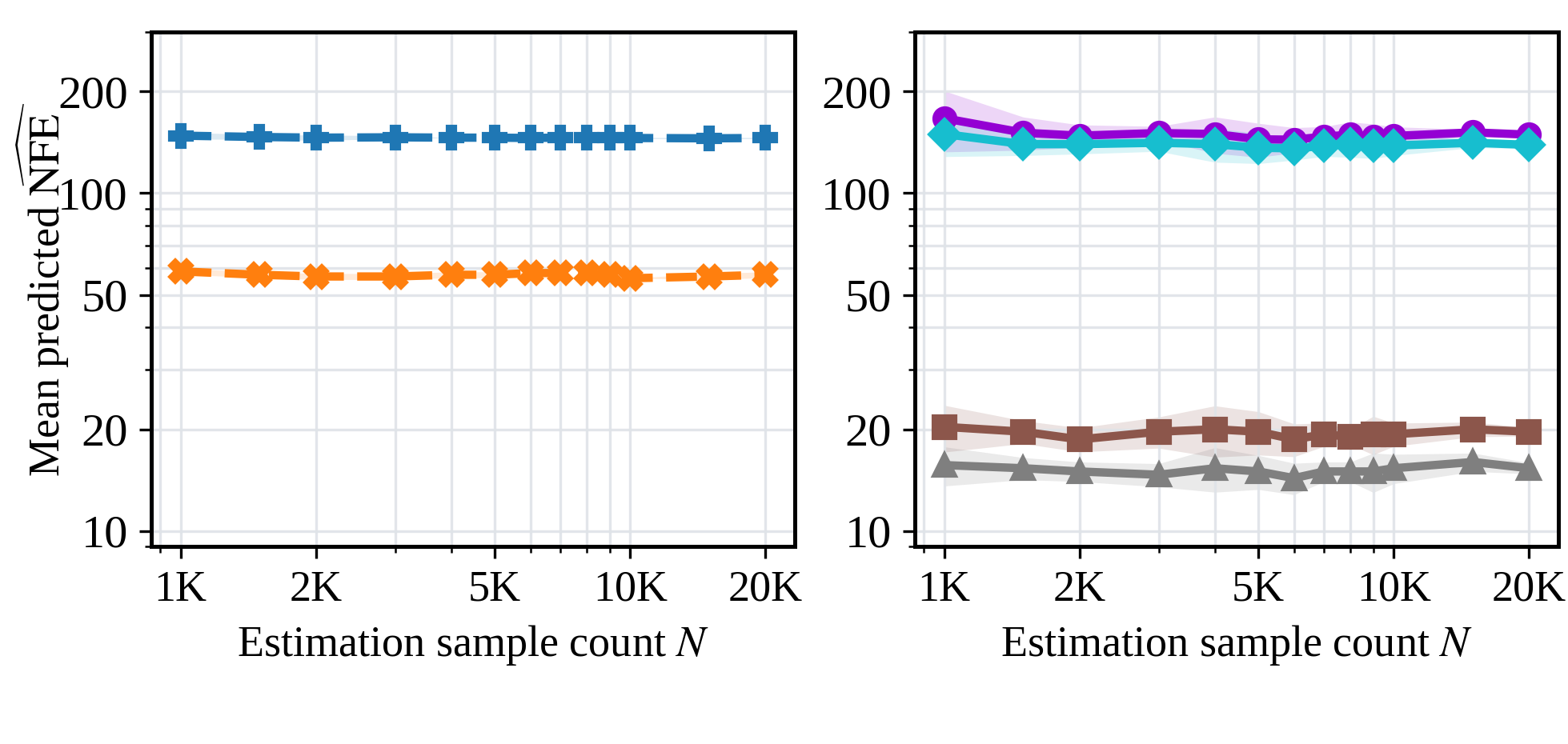}
        \vspace{-0.9cm}
        \caption{Stability of $\widehat{\mathrm{NFE}}$ versus probe count $N$
        at $\tau=0.010$, showing means over three seeds with shaded 95\%
        confidence intervals.}
        \label{fig:budget-stability}
    \end{subfigure}

    \begin{subfigure}{\linewidth}
        \centering
        \includegraphics[width=\linewidth]{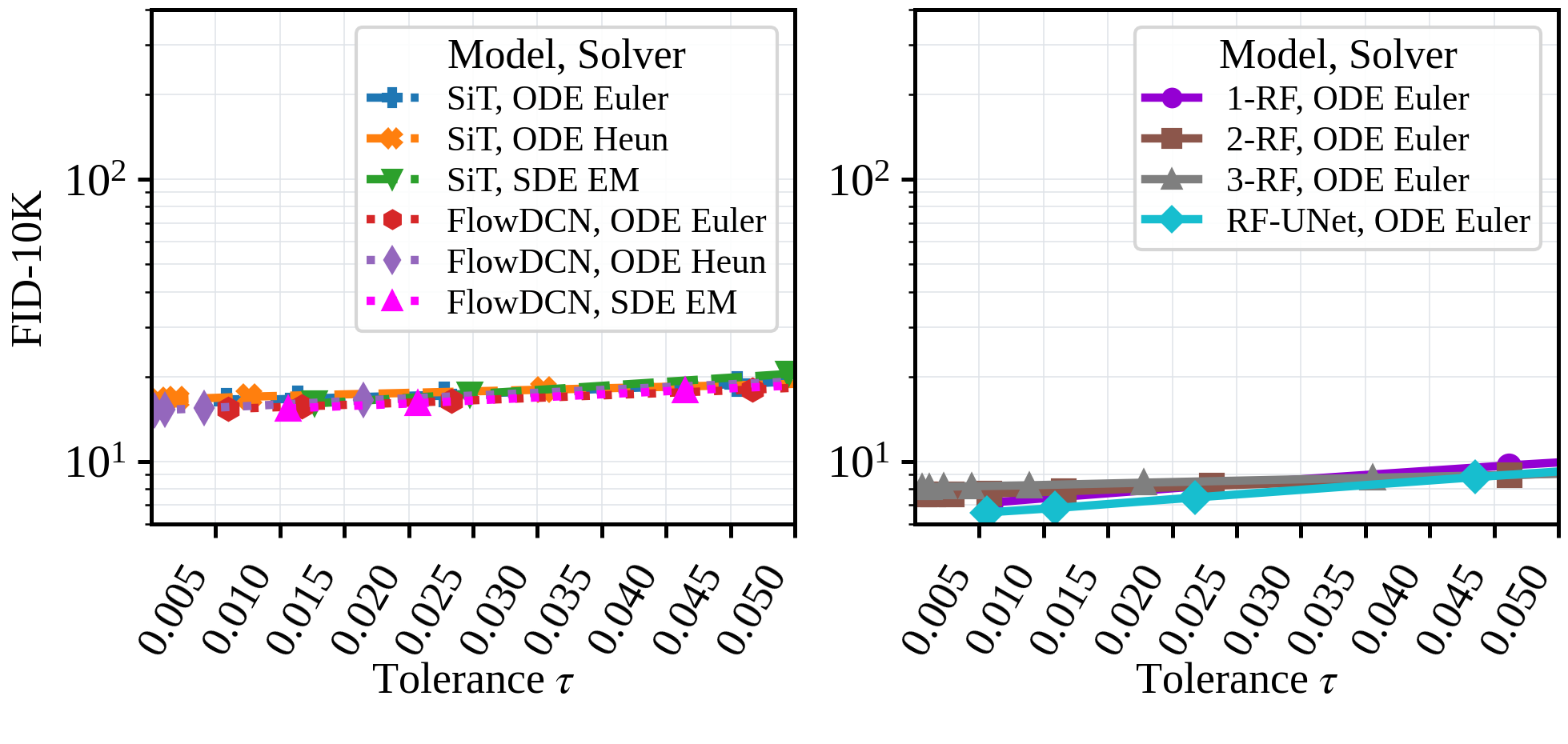}
        \vspace{-0.7cm}
        \caption{FID-10K as a function of the tolerance $\tau$.}
        \label{fig:fid-vs-tau}
    \end{subfigure}

    \caption{
    Analysis of \FB{} budget estimation.
    }
    \label{fig:budget-analysis}
\end{figure}

\noindent \textbf{Stability Analysis.}
Fig.~\ref{fig:budget-stability} shows stable $\widehat{\mathrm{NFE}}$ estimates from $1$K to $20$K probes with narrow $95\%$ confidence bands, supporting use of $2$K probes. As shown in Fig.~\ref{fig:fid-vs-tau}, varying $\tau$ changes the estimated $\widehat{\mathrm{NFE}}$ as expected, while FID varies only modestly around the selected tolerances, indicating that the resulting operating point is robust to small changes in $\tau$.

\noindent \textbf{Computational Cost Analysis.}
For 3-RF, constructing the FID--NFE curve requires generating $10$K samples at six budgets, for a total of $110.3$ minutes. In contrast, \FB{} estimates $\widehat{\mathrm{NFE}}=15$ from $2$K probes and evaluates this  point, requiring $24.3$ minutes. This yields a $4.5\times$ reduction in wall-clock time, saving $86.0$ minutes. Since 3-RF is evaluated only up to $\mathrm{NFE}=32$, the advantage should grow for models requiring broader $\mathrm{NFE}$ sweeps, where generation and FID evaluation must be repeated at additional budgets. 

\section{Conclusion}
\FB{} estimates a fixed solver's numerical resolution from scale-normalized local discretization error, without an NFE--quality sweep. The Euler, Heun, and additive-noise Euler--Maruyama laws are numerical statements; for ODE solvers, the normalized coefficients also control terminal integration error under a bounded-velocity assumption. 

\noindent \textbf{Limitations and Future Work.}
A limitation is that interpolation probes use training-distribution samples and only approximate the sampler trajectory law. Future work will study calibration between $\tau$ and perceptual metrics such as FID.

\bibliographystyle{IEEEbib}
\bibliography{refs}
\end{document}